\documentclass{article}

\usepackage{microtype}

\usepackage{hyperref}

\usepackage{minitoc}
\usepackage{hyperref}
\usepackage{url}
\usepackage{booktabs} 
\usepackage{listings}
\usepackage{xcolor}

\definecolor{codegreen}{rgb}{0,0.6,0}
\definecolor{codegray}{rgb}{0.5,0.5,0.5}
\definecolor{codepurple}{rgb}{0.58,0,0.82}
\definecolor{backcolour}{rgb}{0.95,0.95,0.95}

\lstdefinestyle{mystyle}{
    backgroundcolor=\color{backcolour},   
    commentstyle=\color{codegreen},
    keywordstyle=\color{codegreen},
    numberstyle=\tiny\color{codegray},
    basicstyle=\ttfamily\footnotesize,
    breakatwhitespace=false,         
    breaklines=true,                 
    showspaces=false,                
    showstringspaces=false,
    showtabs=false,                  
    tabsize=2
}

\usepackage[accepted]{arxiv}
\usepackage{amsmath}
\usepackage{amssymb}
\usepackage{mathtools}
\usepackage{amsthm}

\usepackage[capitalize,noabbrev]{cleveref}

\usepackage{subcaption}
\usepackage[font={small}]{caption}
\usepackage{pifont}
\usepackage{makecell}
\usepackage{graphicx}
\usepackage{tabularx}
\usepackage{multirow}
\usepackage{tablefootnote}
\usepackage[flushleft]{threeparttable}

\usepackage{textcomp}

\usepackage{amsmath,amsfonts,bm,physics}

\def\Figref#1{Figure~\ref{#1}}

\def\Tabref#1{Table~\ref{#1}}

\def\Secref#1{Section~\ref{#1}}

\def\eqref#1{equation~\ref{#1}}
\def\Eqref#1{Equation~\ref{#1}}

\def\Appref#1{Appendix~\ref{#1}}

\def\1{\bm{1}}

\def\rvh{{\mathbf{h}}}

\def\vx{{\bm{x}}}

\def\vz{{\bm{z}}}

\def\mA{{\bm{A}}}
\def\mB{{\bm{B}}}
\def\mC{{\bm{C}}}
\def\mD{{\bm{D}}}
\def\mE{{\bm{E}}}
\def\mF{{\bm{F}}}
\def\mG{{\bm{G}}}

\def\mK{{\bm{K}}}

\DeclareMathAlphabet{\mathsfit}{\encodingdefault}{\sfdefault}{m}{sl}
\SetMathAlphabet{\mathsfit}{bold}{\encodingdefault}{\sfdefault}{bx}{n}

\def\gN{{\mathcal{N}}}

\newcommand{\E}{\mathbb{E}}

\newcommand{\R}{\mathbb{R}}

\newcommand{\KL}{D_{\mathrm{KL}}}

\usepackage{array}
\usepackage{multirow}
\usepackage{booktabs}
\usepackage{color, colortbl}
\usepackage{makecell}
\usepackage{float}
\usepackage{thm-restate}
\usepackage{diagbox}

\theoremstyle{plain}

\theoremstyle{definition}

\theoremstyle{remark}

\newenvironment{sproof}{%
  \proof}{\endproof}

\definecolor{Gray}{gray}{0.85}
\definecolor{LightCyan}{rgb}{0.88,1,1}

\makeatletter
\usepackage{xspace}
\def\@onedot{\ifx\@let@token.\else.\null\fi\xspace}
\DeclareRobustCommand\onedot{\futurelet\@let@token\@onedot}

\newcommand{\propref}[1]{Proposition~\ref{#1}}

\def\eg{\emph{e.g}\onedot}

\def\ie{\emph{i.e}\onedot}

\newcommand{\mypara}[1]{\textbf{#1}}

\newcommand{\model}{LS4}

\usepackage[textsize=tiny]{todonotes}

\icmltitlerunning{Deep Latent State Space Models for Time-Series Generation}

\begin{document}

\twocolumn[
\icmltitle{Deep Latent State Space Models for Time-Series Generation}



\icmlsetsymbol{equal}{*}

\begin{icmlauthorlist}
\icmlauthor{Linqi Zhou}{stanford}
\icmlauthor{Michael Poli}{stanford}
\icmlauthor{Winnie Xu}{uoft}
\icmlauthor{Stefano Massaroli}{mila}
\icmlauthor{Stefano Ermon}{stanford}
\end{icmlauthorlist}


\icmlaffiliation{stanford}{Stanford University, @stanford.edu}
\icmlaffiliation{uoft}{University of Toronto}
\icmlaffiliation{mila}{MILA}

\icmlcorrespondingauthor{Linqi Zhou}{linqizhou@stanford.edu}
\icmlcorrespondingauthor{Michael Poli}{poli@stanford.edu}

\icmlkeywords{Machine Learning, ICML}

\vskip 0.3in
]



\printAffiliationsAndNotice{\icmlEqualContribution} 

\begin{abstract}

 Methods based on ordinary 
 differential equations 
 (ODEs) 
 are widely used to build generative models of time-series. In addition to high computational overhead due to explicitly computing hidden states recurrence, existing ODE--based models fall short in learning sequence data with sharp transitions -- common in many real-world systems -- due to numerical challenges during optimization. 
In this work, we propose \model{}, a generative model for sequences with latent variables evolving according to a state space ODE to increase modeling capacity. 
  Inspired by recent  deep state space models (S4), we achieve speedups by leveraging a convolutional representation of \model{} which bypasses the explicit evaluation of hidden states. We show that \model{} significantly outperforms previous continuous-time generative models in terms of marginal distribution, classification, and prediction scores on real-world datasets in the Monash Forecasting Repository, and is capable of modeling highly stochastic data with sharp temporal transitions. \model{} sets state--of--the--art for continuous--time latent generative models, with significant improvement of mean squared error and tighter variational lower bounds on irregularly--sampled datasets, while also being $\times 100$ faster than other baselines on long sequences.

\end{abstract}

\section{Introduction}

Time series are a ubiquitous data modality, and find extensive application in weather \citep{hersbach2020era5} engineering disciplines, biology \citep{peng1995quantification}, and finance \citep{poli2019wattnet}.
The main existing approaches for deep generative learning of temporal data can be broadly categorized into autoregressive \citep{oord2016wavenet}, latent autoencoder models \citep{chen2018neural,yildiz2019ode2vae,rubanova2019latent}, normalizing flows \citep{de2020normalizing}, generative adversarial \citep{yoon2019time,yu2022generating,brooks2022generating}, and diffusion \citep{rasul2021autoregressive}. 
Among these, continuous-time methods (often based on underlying ODEs) are the preferred approach for irregularly-sampled sequences as they can predict at arbitrary time steps and can handle sequences of varying lengths. Unfortunately, existing ODE--based methods \citep{rubanova2019latent,yildiz2019ode2vae} often fall short in learning models for real-world data (\eg, temperature and rain data that follow very sharp transition dynamics) due to their limited expressivity and numerical instabilities during backward gradient computation \citep{hochreiter1998vanishing, niesen2004global, zhuang2020adaptive}.

A natural way to increase the flexibility of ODE-based models is to increase the dimensionality of their (deterministic) hidden states. State--of--the--art methods explicitly compute hidden states by unrolling the underlying recurrence over time (each time step parametrized by a neural network), incurring in polynomial computational costs which prevent scaling to longer sequences. 

An alternative approach to increasing modeling capacity is to incorporate \emph{stochastic} latent variables into the model, a highly successful strategy in generative modeling \citep{kingma2013auto, chung2015recurrent, song2020score, ho2020denoising}. However, reference models like latent neural ODE models~\citep{rubanova2019latent} inject stochasticity only at the initial condition of the system.
In contrast, we introduce \model{}, a latent generative model where the sequence of latent variables is represented as the solution of linear state space equations \citep{chen1984linear}. Unrolling the recurrence equation shows an autoregressive dependence in the sequence of latent variables, the joint of which is highly expressive in representing time series distributions. The high dimensional structure of the latent space, being equivalent to that of the input sequence, allows \model{} to learn expressive latent representations and fit the distribution of sequences produced by a \textit{family} of dynamical systems, a common setting resulting from non--stationarity. We further show how \model{} can learn the dynamics of \textit{stiff} \citep{shampine2007stiff} dynamical systems where previous methods fail to do so. Inspired by recent works on deep state space models, or stacks of linear state spaces and non-linearities \citep{gu2020hippo,gu2021efficiently}, we leverage a convolutional kernel representation to solve the recurrence, bypassing any explicit computations of hidden states via the recurrence equation, which ensures log--linear scaling in both the hidden state space dimensionality as well as sequence length. 

We validate our method across a variety of time series datasets, benchmarking \model{} against an extensive set of baselines. We propose a set of $3$ metrics to measure the quality of generated time series samples and show that \model{} performs significantly better than baselines on datasets with stiff transitions and obtains on average $30\%$ lower MSE scores and ELBO. On sequences with $\approx 20K$
lengths, our model trains $\times 100$ faster than the baseline methods.

\section{Related Work}
Rapid progress on deep generative modeling of natural language and images has consolidated diffusion \citep{ho2020denoising,song2020score,song2019generative,sohl2015deep} and autoregressive techniques \citep{brown2020language} as the state--of--the--art. Although various approaches have been proposed for generative modeling of time series and dynamical systems, consensus on the advantages and disadvantages of each method has yet to emerge.

\mypara{Deep generative modeling of sequences.}
All the major paradigms for deep generative modeling have seen applications to time series and sequences. Prior works for time-series generation have adopted VAE-, Flow-, and GAN-based approaches \citep{chung2015recurrent, deng2020modeling, yu2017seqgan, yoon2019time} which utilize recurrent architectures to keep track of internal states. For continuous--time data, other works combine Gaussian processes \citep{fortuin2020gp} or ODEs \citep{yildiz2019ode2vae, rubanova2019latent} into their probabilistic frameworks and show promising results for time-series extrapolation and generation. 
Other works on time-series have adopted diffusion-based frameworks \citep{tashiro2021csdi}, but similar to the methods requiring recurrent computations, these methods suffer from prolonged generation time. 
Among these methods, most relevant to our work are latent continuous--time autoencoder models proposed by \citet{chen2018neural,yildiz2019ode2vae,rubanova2019latent}, where a neural differential equation encoder is used to parameterize as distribution of initial conditions for the decoder. \citet{massaroli2021differentiable} proposes a variant that parallelizes computation in time by casting solving the ODE as a root finding problem. Beyond latent models, other continuous--time approaches are given in \citet{kidger2020neural}, which develops a GAN formulation using SDEs. 

\mypara{State space models.}
\textit{State space models} (SSMs) are at the foundation of dynamical system theory \citep{chen1984linear} and signal processing \citep{oppenheim1999discrete}, and have also been adapted to deep generative modeling. \citet{chung2015recurrent, bayer2014learning} propose VAE variants of discrete--time RNNs, generalized later by \citep{franceschi2020stochastic}, among others. However, these models all unroll the recurrence equation and are thus challenging to scale to longer sequences. 

Our work is inspired by recent advances in deep architectures built as stacks of linear SSMs, notably S4 \citep{gu2021efficiently, gu2020hippo}. 
The HiPPO-initialized and deeply stacked linear SSMs have shown promising results for modeling long sequences with unprecedented efficiency, and they have shown remarkable modeling capacity for capturing long-range dependencies across large time scale. 
Similar to S4, our generative model leverages the convolutional representation of SSMs during training and inference, thus bypassing the need to materialize the hidden state of each recurrence. This allows us to speed up training and inference by a large margin compared to prior generative methods. More importantly, we augment the deep SSMs to model sequential latent variables, which increase the capacity to capture more complex temporal dynamics (\eg stiff transitions) and achieve better generation results.

\section{Preliminaries}

We briefly introduce relevant details of continuous-time SSMs and their different representations. Then we introduce preliminaries of generative models for sequences.

\subsection{State Space Models (SSM)} \label{sec:ssm}

A \textit{single-input single-output (SISO)} linear state space model is defined by the following differential equation
\begin{align}\label{eq:ssm}
    \begin{split}
        \frac{d}{dt}\bm{h}_t &= \mA \bm{h}_t + \mB x_t\\
        y_t &= \mC \bm{h}_t + \mD x_t
    \end{split}
\end{align}

with scalar \textit{input} $x_t\in\R$, \textit{state} $\bm{h}_t\in\R^N$ and scalar \textit{output} $y_t\in\R$. The system is fully characterized by the matrices $\mA \in \R^{N\times N}, \mB \in \R^{N\times 1}, \mC \in \R^{1\times N}, \mD \in \R^{1\times 1}$. Let $x, y\in \mathcal C([a, b], \R)$ be absolutely continuous real signals on time interval $[a, b]$. Given an initial condition $\bm{h}_0\in\R^N$ the SSM (\ref{eq:ssm}) realizes a mapping $x\mapsto y$.

SSMs are a common tool for processing continuous input signals. We consider \textit{single input single output} (SISO) SSMs, noting that input sequences with more than a single channel can be processed by applying multiple SISO SSMs in parallel, similarly to regular convolutional layers. 
We use such SSMs as building blocks to map each input dimension to each output dimension in our generative model.

\mypara{Discrete recurrent representation.} 
In practice, continuous input signals are often sampled at time interval $\Delta$ and the sampled sequence is represented by $x=(x_{t_0}, x_{t_1}, \dots, x_{t_L})$ where $t_{k+1} =t_{k}+ \Delta$. The discretized SSM follows the recurrence
\begin{align}\label{eq:dis-ssm}
\begin{split}
    \bm{h}_{t_{k+1}} &= \bar{\mA } \bm{h}_{t_k} + \bar{\mB } x_{t_k}\\
    y_{t_k} &= \mC \bm{h}_{t_k} + \mD x_{t_k}
\end{split}
\end{align}
where $\bar{\mA } = e^{\mA \Delta}$, $\bar{\mB } = \mA ^{-1}(e^{\mA \Delta}-I)\mB$ with the assumption that signals are constant during the sampling interval. 

Among many approaches to efficiently computing $e^{\mA \Delta}$, \citet{gu2021efficiently} use a bilinear transform to estimate $e^{\mA \Delta} \approx (I - \frac{1}{2}\mA \Delta)^{-1}(I + \frac{1}{2}\mA \Delta)$.

This recurrence equation can be used to iteratively solve for the next hidden state $h_{t_{k+1}}$, allowing the states to be calculated like an RNN or a Neural ODE \citep{chen2018neural,massaroli2020dissecting}. 

\mypara{Convolutional representation.} Recurrent representations of SSM are not practical in training because explicit calculation of hidden states for every time step requires $\mathcal{O}(N^2L)$ in time and $\mathcal{O}(NL)$ in space for a sequence of length $L$\footnote{Further explanations in \Appref{app:der-recurr}}. This materialization of hidden states significantly restricts RNN-based methods in scaling to long sequences. To efficiently train an SSM, the recurrence equation can be fully unrolled, assuming zero initial hidden states, as
\begin{align*}
   &\bm{h}_{t_0} = \bar{\mB } x_{t_0}  \hspace{5pc}  \bm{h}_{t_1} = \bar{\mA }\bar{\mB } x_{t_1} + \bar{\mB } x_{t_0}\\
    &y_{t_0} = \mC \bar{\mB } x_{t_0} \hspace{4.5pc} y_{t_1} = \mC \bar{\mA }\bar{\mB } x_{t_1} + \mC \bar{\mB } x_{t_0}\\[3mm]
      &\bm{h}_{t_2} = \bar{\mA }^{2} \bar{\mB } x_{t_2} +  \bar{\mA }\bar{\mB } x_{t_1} + \bar{\mB } x_{t_0} \hspace{4pc} \dots\\
 &y_{t_2} = \mC \bar{\mA }^{2} \bar{\mB } x_{t_2} +  \mC \bar{\mA }\bar{\mB } x_{t_1} + \mC \bar{\mB } x_{t_0}  \hspace{2pc} \dots
\end{align*}
and more generally as, 
\begin{align*}
    y_{t_k} &= \mC \bar{\mA }^{k} \bar{\mB } x_{t_k} +  \mC \bar{\mA }^{k-1}\bar{\mB } x_{k-1} + \dots + \mC \bar{\mB } x_{t_0}
\end{align*}
For an input sequence $x = (x_{t_0}, x_{t_1}, \dots, x_{t_L})$, one can observe that the output sequence $y=(y_{t_0}, y_{t_1}, \dots, y_{t_L})$ can be computed using a convolution with a skip connection
\begin{align}\label{eq:conv-ssm}
    y =\mC  &\mK * x + \mD x,\\ &\text{ where } \mK = ( \bar{\mB }, \bar{\mA }\bar{\mB }, \dots, \bar{\mA }^{L-1}\bar{\mB },
\bar{\mA }^{L}\bar{\mB })\nonumber
\end{align}

This is the well-known connection between SSM and convolution \citep{oppenheim1975digital, chen1984linear, chilkuri2021parallelizing, romero2021ckconv, gu2020hippo, gu2021efficiently, gu2022parameterization} 

and it can be computed very efficiently with a Fast Fourier Transform (FFT), which scales better than explicit matrix multiplication at each step.

\subsection{Variational Autoencoder (VAE)} \label{sec:vae}

VAEs \citep{kingma2013auto, burda2015importance} are a highly successful paradigm in learning latent representations of high dimensional data and is remarkably capable at modeling complex distributions. A VAE introduces a joint probability distribution between a latent variable $\vz$ and an observed random variable $\vx$ of the form $$p_{\theta}(\vx, \vz) = p_{\theta}(\vx \mid \vz)p(\vz)$$ where $\theta$ represents learnable parameters.

The prior $p(\vz)$ over the latent is usually chosen as a standard Gaussian distribution, and the conditional distribution $p_{\theta}(\vx \mid \vz)$ is defined through a flexible non-linear mapping (such as a neural network) taking $\vz$ as input. 

Such highly flexible non-linear mappings often lead to an intractable posterior $p_\theta(\vz \mid \vx)$. Therefore, an inference model with parameters $\phi$ parametrizing $q_{\phi}(\vz \mid \vx)$ is introduced as an approximation which allows learning through a variational lower bound of the marginal likelihood:
\begin{align}
    \log p_\theta(\vx) \ge -\KL(&q_{\phi}(\vz \mid \vx) \| p(\vz)) \\ \nonumber
        &+ \E_{q_{\phi}(\vz \mid \vx)}\left[\log p_{\theta}(\vx \mid \vz)\right]
\end{align}
where $\KL(\cdot \| \cdot)$ is the Kullback-Leibler divergence between two distributions.

\mypara{VAE for sequences.} Sequence data can be modeled in many different ways since the latent space can be chosen to encode information at different levels of granularity, \ie $\vz$ can be a single variable encoding entire trajectories or a sequence of variables of the same length as the trajectories. We focus on the latter.

Given observed sequence variables $\vx_{\le T}$ up to time $T$ discretized into sequence $(\vx_{t_0}, \dots, \vx_{t_{L-1}})$ of length $L$ where $t_{L-1} = T$, a sequence VAE model with parameters $\theta, \lambda, \phi$ learns a generative and inference distribution 
\resizebox{1.02\linewidth}{!}{
  \begin{minipage}{\linewidth}
\begin{align*}
    &p_{\theta, \lambda}(\vx_{\le  t_{L-1}}, \vz_{\le  t_{L-1}})=\prod_{i=0}^{L-1} p_\theta(\vx_{t_i} \mid \vx_{< t_i}, \vz_{\le t_i})p_\lambda(\vz_t \mid \vz_{<t_i})\\
    &q_\phi(\vz_{\le t_{L-1}} \mid \vx_{\le t_{L-1}}) = \prod_{i=0}^{L-1} q_\phi(\vz_{t_i} \mid \vx_{\le t_{t_i}})
\end{align*}
  \end{minipage}
}
where $\vz_{\le t_{L-1}} = (\vz_{t_0}, \dots, \vz_{t_{L-1}})$ is the corresponding latent variable sequence. The approximate posterior $q_\phi$ is explicitly factorized as a product of marginals due to efficiency reasons we shall discuss in the next section. Given this form of factorization, the variational lowerbound has been considered for discrete sequence data \citep{chung2015recurrent} by minimizing the objective
\resizebox{0.98\linewidth}{!}{
  \begin{minipage}{\linewidth}
\begin{align}
\begin{split}
    \E_{q_\phi(\vz_{\le t_{L-1}} \mid \vx_{\le t_{L-1}})} \Big[ &\sum_{i=0}^{L-1}  \KL(q_\phi(\vz_{t_i} \mid \vx_{\le t_i}) \| p_\lambda(\vz_i \mid \vz_{<t_i})) \\&- \log p_\theta(\vx_{t_i} \mid \vx_{< {t_i}} \vz_{\le {t_i}}) \Big]
    \end{split}\label{eq:seq-elbo}\raisetag{15pt}
\end{align}
  \end{minipage}
}

Our model also trains by this objective. After training, the generative model can then sample $z_t$ from the prior $p_{\lambda}$ autoregressively and given the sampled $z_t$, each $x_t$ can be sampled autoregressively using $p_\theta(\vx_{t_i} \mid \vx_{< {t_i}} \vz_{\le {t_i}})$. 
\section{Method}

In this section, we introduce \textit{Latent S4} (\model{}), a latent variable generative model parameterized using SSMs. We show how SSMs can parametrize the generative distribution $p_\theta(\vx_{\le T} | \vz_{\le T}) p_\lambda(\vz_{\le T})$, the prior distribution $p_\lambda (\vz_{\le T})$ and the inference distribution $q_\phi(\vz_{\le T} \mid \vx_{\le T})$ effectively. 
For the purpose of exposition, we can assume $z_t, x_t$ are scalars at any time step $t$. Their generalization to arbitrary dimensions is discussed in \Secref{sec:prop}.

We first define a structured state space model with two input streams and use this as a building block for our generative model. It is an SSM of the form
\begin{align*}\label{eq:multi-ssm}
        \frac{d}{dt}\bm{h}_t &= \mA \bm{h}_t + \mB x_t + \mE z_t\\
        y_t &= \mC \bm{h}_t + \mD x_t + \mF z_t
\end{align*}
where $x, y, z\in \mathcal C([0,T], \R)$ are continuous real signals on time interval $[0,T]$. We denote $H(x,z,\mA,\mB,\mE,\bm{h}_0,t) =H_\beta(x,z,\bm{h}_0,t) $, where $\beta$ denotes trainable parameters $\mA, \mB,\mE$, as the deterministic function mapping from signals $x,z$ to $\bm{h}_t$, the state at time $t$, given initial state $\bm{h}_0$ at time $0$.
The above SSM can be compactly represented by 
\begin{align}
    y_t = \mC H_\beta(x,z,\bm{h}_0,t) + \mD x_t + \mF z_t
\end{align}
When the continuous-time input signals are discretized into discrete-time sequences $(x_{t_0}, \dots, x_{t_{L-1}})$ and $(z_{t_0}, \dots, z_{t_{L}})$, the corresponding hidden state at time $t_k$ has a convolutional view (assuming $\mD=\mF=\bm{0}$ for simplicity)
\begin{align}\label{eq:conv-multi-ssm}
\begin{split}
    y_{t_k} = \mC&\mK_{t_{k}} * x_{t_k} +  \mC \hat\mK_{t_{k}} * z_{t_k}, \\&\text{ where }\; \mK_{t_{k}} = \bar{\mA}^{k}\bar{\mB},~~\hat\mK_{t_{k}} = \bar{\mA}^{k}\bar{\mE} \nonumber
\end{split}
\end{align}
which can be evaluated efficiently using FFT. Additionally, $\mA$ is HiPPO-initialized \citep{gu2021efficiently} for all such SSM blocks.

\subsection{Latent Space as Structured State Space}
The goal of the prior model is to realize a sequence of random variables $(z_{t_0}, z_{t_1}, \dots, z_{t_{L}})$, which the prior distribution $p_\lambda (z_{\le t_{L}})$ models autoregressively. 
Suppose $(z_{t_0}, z_{t_1}, \dots, z_{t_n})$ is a sequence of latent variables up to time $t_n$, we define the prior distribution of $z_{t_n}$ autoregressively as 
\begin{align}
p_\lambda(z_{t_n} \mid z_{<t_n}) = \gN(\mu_{z,n}(z_{<t_n},\lambda), \text{$\sigma_{z,n}^2(z_{<t_n},\lambda)$}) 
\end{align}
where the mean $\mu_{z,n}$ and standard deviation $\sigma_{z,n}$ are deterministic functions of previously generated $z_{<t_n}$ parameterized by $\lambda$. To parameterize the above distribution, we first define an intermediate building block, a stack of which will produce the wanted distribution.

\mypara{\model{} prior block.} The forward pass through our SSM is a two--step procedure: first, we consider the latent dynamics of $z$ on $[t_0, t_{n-1}]$ where we simply leverage \Eqref{eq:multi-ssm} to define the hidden states to follow $H_{\beta_1} (0,z,0,t)$. Second, on $(t_{n-1}, t_n]$, since no additional $z$ is available in this interval, we ignore additional input signals in the ODE and only include the last given latent, \ie $z_{t_{n-1}}$, as an auxiliary signal for the outputs, which can be compactly denoted, with a final GELU non-linearity, as
\begin{align}\label{eq:prior_layer}
    \begin{split}
        y_{z, n} &= \text{GELU}(\mF_{y_z} z_{t_{n-1}} \\& +
        \mC_{y_z} H_{\beta_1}(0,0,\underbrace{H_{\beta_2}(0,z_{[t_0,t_{n-1}]},h_{t_{n-1}},\mathbf{0},t_{n-1})}_{\bm{h}_{t_{n-1}}},t_n) )
    \end{split}
\end{align}
 Output $y_{z, n}$ has the same dimensionality as each $z_t$ and is a function of all $z_{<t_n}$ and we will use it to build towards modeling the distribution of $z_{t_n}$. We call the above equation \textit{\model{} prior layer} and we define below our \textit{\model{} prior block}, which is built upon a ResNet structure with a skip connection, denoted as
\begin{align}
\begin{split}
    \text{\model{}}_{\text{prior}}( z_{[t_0,t_{n-1}]}, \psi) = \text{LN}(\mG_{y_z} y_{z, n} + b_{y_z}) + z_{t_{n-1}}
\end{split}
\end{align}
where LN denotes LayerNorm and $\psi$ denotes the union of parameters $\beta_i, \mC_{y_z}, \mF_{y_z}, \mG_{y_z}, b_{y_z}$. We define the final parameters $\mu_{z,n}$ and $\sigma_{z,n}$ for the conditional distribution in the autoregressive model as the result of a stack of \textit{\model{} prior blocks}. Specifically, the input $z_t$'s are input into a stack of $B$ blocks and at the final layer two separate blocks branch out to separately parameterize $\mu_{z,n}$ and $\sigma_{z,n}$. During generation, as an initial condition, $z_{t_0} \sim \gN(\mu_{z,0}, \sigma_{z,0}^2)$ where $\mu_{z,0},  \sigma_{z,0}$ are learnable parameters, and subsequent latent variables are generated autoregressively. We specify our architecture in \Appref{app:arch} and use $\lambda$ to denote the union of all trainable parameters.

\subsection{Generative Model}
 Given the latent variables, we now specify a decoder that represents the distribution $p_\theta(x_{\le t_{L}} | z_{\le  t_{L}})$. Suppose $z_{\le t_{L}}$ is a latent path generated via the latent state space model, the output path $x_{\le t_{L}}$ also follows the state space formulation. Assuming we have generated $(x_{t_0}, \dots, x_{t_{n-1}})$ and $(z_{t_0}, \dots, z_{t_n})$, the conditional distribution of $x_{t_n}$ is parametrized as
\begin{align}
    p_\theta(x_{t_n} | x_{<t_n}, z_{\le t_n}) = \gN(\mu_{x,n}(x_{<t_n}, z_{\le t_n}, \theta), \sigma_{x}^2)
\end{align}
where $\sigma_{x}$ is a pre-defined observation standard deviation and $\mu_{x,n}$ is a deterministic function of $z_{\le t_n}$ and $x_{<x_n}$. 

\mypara{\model{} generative block.} Different from the prior block, both observation and latent sequences are input into our model, and we define intermediate outputs $g_{x, n}$ and $g_{z,n}$ as
\begin{align}
    \begin{split}
        h_{t_n} = &H_{\beta_3}(0,z_{t_{n-1}},H_{\beta_4}(x_{[t_0,t_{n-1}]},z_{[t_0,t_{n-1}]},\mathbf{0},t_{n-1}),t_n)\\
        &g_{x, n} =\text{GELU}(\mC_{g_x} h_{t_n} + \mD_{g_x} x_{t_{n-1}} + \mF_{g_x} z_{t_{n}})  \\
        &g_{z, n} =\text{GELU}(\mC_{g_z} h_{t_n}  + \mD_{g_z} x_{t_{n-1}} + \mF_{g_z} z_{t_{n}})
    \end{split}
\end{align}
which are used to build a \textit{\model{} generative block} defined as
\begin{align}
\begin{split}
        &\hat{g}_{x, n} = \text{LN}(\mG_{g_x} g_{x, n} + b_{g_x}) + x_{t_{n-1}}\\
    &\hat{g}_{z, n} = \text{LN}(\mG_{g_z} g_{z, n} + b_{g_z}) + z_{t_{n}}\\
    &\text{\model{}}_{\text{gen}}( x_{[t_0,t_{n-1}]}, z_{[t_0,t_{n}]}, \psi) = (\hat{g}_{x, n}, \hat{g}_{z, n})
\end{split}
\end{align}
where $\psi$ denotes all parameters inside the block. Note that the generative block gives two streams of outputs each having the same ResNet-like structure as in the prior model, and the output of our generative block can be used as inputs for the next stack. We then define the final mean value $\mu_{x,n}$ as the result of a stack of \textit{\model{} generative blocks}. The initial condition for generation is given as $x_{t_0}\sim \gN(\mu_{x,0}(z_0, \theta), \sigma_x)$ where $\mu_{x, 0}$ exactly follows our formulation while taking only $z_{t_0}$ as input. The subsequent $x_{t_n}$'s are generated autoregressively. We specify our architecture in \Appref{app:arch} and use $\theta$ to denote the union of all trainable parameters.

\subsection{Inference model}
The latent variable model up to time $t_n$ has intractable posterior $p_\theta(z_{\le t_n} \mid x_{\le t_n})$. Therefore, we approximate this distribution with $q_\phi(z_{\le t_n} \mid x_{\le t_n})$ using variational inference. 

We parameterize the inference distribution at time $t_n$ to depend only on the observed path $x_{\le t_n}$:
\begin{align}
    q_\phi(z_t \mid x_{\le t_n}) = \gN(\hat{\mu}_{z,t_n}(x_{\le t_n}, \phi), \text{$\hat{\sigma}_{z,t_n}^2(x_{\le t_n}, \phi)$})   
\end{align}
This choice of dependency is as noted in our objective (\Eqref{eq:seq-elbo}). By having each $z_t$ explicitly depending on $x_{\le t_n}$ only, we obviate the need for explicit recurrence to obtain $z_{t_n}$. We can then leverage the fast convolution operation to obtain all $z_t$ in parallel, thus achieving fast inference time, in contrast to the autoregressive nature of the prior and generative model.

\mypara{\model{} inference block.} The inference block is defined as 
\begin{align}
    \begin{split}
    &\hat{y}_{z, n} = \text{GELU}(\mC_{\hat{y}_z} H_{\beta_5}(x_{[t_0,t_{n}]},0,\mathbf{0},t_{n-1}) + \mD_{\hat{y}_z} x_{t_{n}})\\
        &\text{\model{}}_{\text{inf}}( x_{[t_0,t_{n}]}, \psi) = \text{LN}(\mG_{\hat{y}_z} \hat{y}_{z, n} + b_{\hat{y}_z}) + x_{t_{n}}
    \end{split}
\end{align}
Notice that input $x$ is fully present in $[t_{0}, t_n]$ unlike in the generative model. Similar to the prior counterparts, $\hat{\mu}_{z,t}$ and $\hat{\sigma}_{z,t}$ are obtained by first feeding $x_t$'s into a stack of inference blocks where the final block branches out to separately model the mean and variance. Due to the convolutional nature of our inference model, the training and inference can be done very efficiently, as will be demonstrated in the next and the experiment section.

\subsection{\model{}: Properties and Practice}\label{sec:prop}

We highlight some properties of \model{}. In particular, we compare in the following proposition the expressive power of our generative model against structured state space models.
\begin{restatable}[]{proposition}{subsume}\label{prop:subsume} (LS4 subsumes S4.)
Given any autoregressive model $r(x)$ with conditionals $r(x_n | x_{<n})$ parameterized via deep S4 models, there exists a choice of $\theta, \lambda, \phi$ such that $p_{\theta, \lambda}(x) = r(x)$ and $p_{\theta, \lambda}(z|x) = q_\phi(z|x)$, i.e. the variational lower bound (ELBO) is tight.
\end{restatable}
A proof sketch is provided in \Appref{app:proof}. 
This result shows that \model{} subsumes autoregressive generative models based on vanilla S4 \citep{gu2021efficiently}, given that the architecture between SSM layers is the same. Crucially, with the assumption that we are able to globally optimize the ELBO training objective, LS4 will fit the data at least as well as vanilla S4.

\mypara{Scaling to arbitrary feature dimensions.} 
So far we have assumed the input and latent signals are real numbers. The approach can be scaled to arbitrary dimensions of inputs and latents by constructing \model{} layers for each dimension which are input into a mixing linear layer. We call such parallelized SSMs \textit{heads} and provide a pseudo-code in \Appref{app:arch}.

\begin{restatable}[]{proposition}{efficiency}\label{prop:efficiency}
    (Efficiency.) For a SSM with $H$ heads, an observation sequence of length $L$ and hidden dimension $N$ can be calculated in $\mathcal{O}(H(L+N)\log(L+N))$ time and $\mathcal{O}(H(L+N))$ space. 
\end{restatable}

We provide proof in \Appref{app:proof}. Note that our model is much more efficient in both time and space than RNN/ODE-based methods (which requires $\mathcal{O}(N^2 L)$ in time and $\mathcal{O}(NL)$ in space as discussed in \Secref{sec:ssm}). To demonstrate the computation efficiency, we additionally provide below pseudo-code for a single \model{} prior layer \ref{eq:prior_layer}. The other blocks can be similarly constructed. 
\begin{lstlisting}[language=Python]
  def (*@\textcolor{blue}{LS4\_prior\_layer}@*)(z,  A, B, C, F, h_0): 
     # z: (B, L, 1) 
     K = C @ materialize_kernel(z, A, B, h_0) # O((L+N)log(L+N)) time
     CH = fft_conv(K, z) 
            # O(LlogL) time and O(L) space
     y = gelu(CH + F * z) 
     return y
\end{lstlisting}
Note that in practice, $\mA$ is HiPPO initialized \citep{gu2020hippo} and the materialized kernel includes $\mC$ so that the convolution is computed directly in the projected space, bypassing materializing the high-dimensional hidden states. 

\mypara{Parametrizing the initial conditions for generation.} In the previous sections we mentioned that generation starts by sampling with learnable parameters $\mu_{\boldsymbol{\cdot},0}$ and $\sigma_{\boldsymbol{\cdot},0}$. These parameters can be instead given as the output from the same networks that produce each $\mu_{\boldsymbol{\cdot},<0}$ and $\sigma_{\boldsymbol{\cdot},<0}$ with input as either a learnable parameter or simply the zero vector.

\begin{table*}[t]
\footnotesize
\resizebox{\linewidth}{!}{
    \begin{tabular}{ll|cccccccc}
    \toprule
    Data & Metric & RNN-VAE & GP-VAE & ODE$^2$VAE & Latent ODE & TimeGAN & SDEGAN & SaShiMi & \textbf{\model{} (Ours)} \\
    \midrule
    FRED-MD & Marginal $\downarrow$ & 0.132  & 0.152 & 0.122& 0.0416& 0.0813& 0.0841& 0.0482 &  \cellcolor{blue!10}\textbf{0.0221} \\
    & Class. $\uparrow$ & 0.0362 & 0.0158 & 0.0282 & 0.327 & 0.0294 &0.501 &0.00119 & \cellcolor{blue!10}\textbf{0.544} \\
    & Prediction $\downarrow$ &  1.47 & 2.05 &  0.567 & \cellcolor{blue!10}\textbf{0.0132} & 0.0575 & 0.677  & 0.232 & 0.0373\\
    \midrule
    NN5 Daily & Marginal $\downarrow$ & 0.137 &  0.117&  0.211 &  0.107   & 0.0396  &  0.0852  &0.0199 &  \cellcolor{blue!10}\textbf{0.00671}\\
    & Class. $\uparrow$ & 0.000339 &  0.00246 &  0.00102 & 0.000381& 0.00160& 0.0852  & 0.0446 & \cellcolor{blue!10}\textbf{0.636}  \\
    & Prediction $\downarrow$ &0.967   & 1.169  & 1.19  &1.04 &  1.34 &  1.01 & 0.849 & \cellcolor{blue!10}\textbf{0.241} \\
    \midrule 
    Temp Rain & Marginal $\downarrow$ &0.0174   &   0.183  &  1.831 & \cellcolor{blue!10}\textbf{0.0106}  &  0.498 &  0.990 &0.758 &  0.0834 \\
    & Class. $\uparrow$& 0.00000212 & 0.0000123 & 0.0000319& 0.0000419&0.00271  &  0.0169  & 0.0000167 & \cellcolor{blue!10}\textbf{0.976} \\
    & Prediction $\downarrow$ &  159  &  2.305 & 1.133   & 145 &  1.96   & 2.46 & 2.12  & \cellcolor{blue!10}\textbf{0.521}  \\
    \midrule 
    Solar Weekly & Marginal $\downarrow$ & 0.0903&0.308    &  0.153 &0.0853  &  0.0496  & 0.147  & 0.173 &\cellcolor{blue!10}\textbf{0.0459} \\
    & Class. $\uparrow$ &   0.0524   & 0.000731 &0.0998  & 0.0521 &0.6489 &  0.591 &0.00102  & \cellcolor{blue!10} \textbf{0.683}  \\
    & Prediction $\downarrow$ &   1.25  & 1.47   & 0.761  & 0.973  &  0.237& 0.976&  0.578& \cellcolor{blue!10}\textbf{0.141}\\
    \bottomrule 
    \end{tabular}
}
\caption{Generation results on FRED-MD, NN5 Daily, Temperature Rain, and Solar Weekly.}\label{tab:gen}
\end{table*}
\section{Experiments}

In this section, we verify the modeling capability of \model{} empirically. There are three main questions we seek to answer: (1) How effective is \model{} in modeling stiff sequence data? (2) How expressive is \model{} in scaling to real time-series with a variety of temporal dynamics? (3) How efficient in training and inference is \model{} in terms of wall-clock time?

\subsection{Learning to generate data from stiff systems}

Modeling data generated by dynamics with widely separated time scales has been proven to be particularly challenging for vanilla ODE-based approaches which make use of standard explicit solvers for inference and gradient calculation. \cite{kim2021stiff} showed that as the learned dynamics \textit{stiffen} up to track data paths, the ODE numerics start to catastrophically fail; the inference cost raises drastically and the gradient estimation process becomes ill--conditioned. These issues can be mitigated by employing implicit ODE solvers or \textit{ad-hoc} rescalings of the learned vector field (see \citep{kim2021stiff} for 
further details). 

In turn, state--space models do not suffer from stiffness of dynamics as the numerical methods are sidestepped in favor of an exact evaluation of the convolution operator. 
We hereby show that LS4 is able to model data generated by a prototype stiff system.

\mypara{FLAME problem} We consider a simple model of flame growth (FLAME) \citep{wanner1996solving}, which has been extensively studied as a representative of highly stiff systems:
\[\frac{d}{dt}x_t = x_t^2 - x_t^p\]
where $p \in \{3,4,\dots,10\}$. For each $p$, $1000$ trajectories are generated for $t \in [0, 1000]$ with unit increment. 

\mypara{Generation.} In \Figref{fig:toy-gen}, we show the mean trajectories and the distribution at each time step and that our samples closely match the ground-truth data. The Latent ODE \citep{rubanova2019latent} instead fails to do so and produces non--stiff samples drastically different from the target.

\mypara{Marginal Distribution.} We plot the marginal distribution of the real data and the generated data from both our model and Latent ODE. Since the stiff transitions are mostly distributed before $10\%$ of total steps, we visualize the marginal histograms at $4$ time steps equally spaced between the $0.5\%$ and $10\%$ steps in log scale (see \Figref{fig:toy-marginal}). We observe that the empirical histogram matches the ground truth distribution significantly better than what is produced by Latent ODEs, as also qualitatively visible from the samples in (a).

\begin{figure}[h]
    \centering
    \begin{subfigure}[t]{0.8\linewidth}
         \centering
        \includegraphics[width=\linewidth]{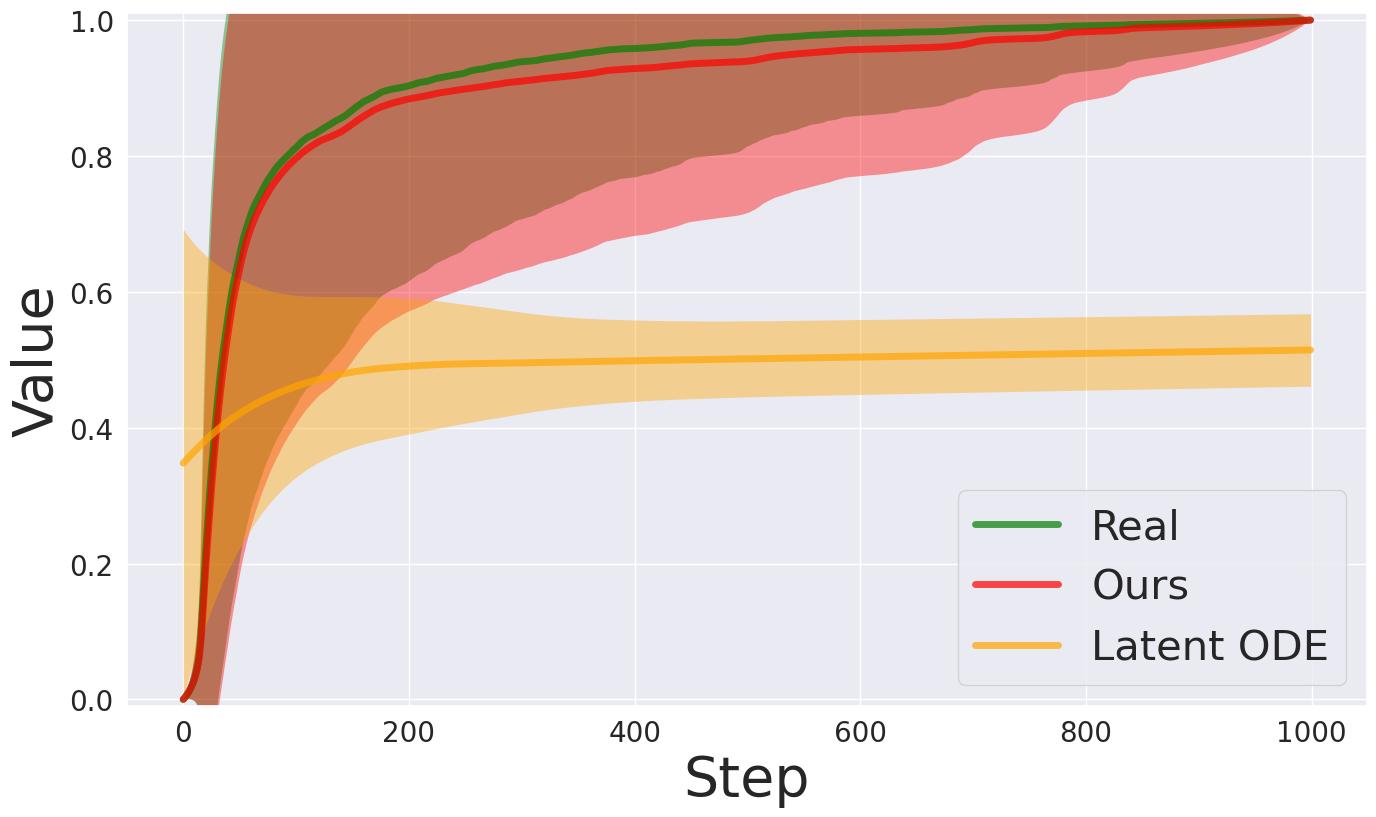}
        \caption{Generation of the stiff system.}
        \vspace{2mm}
        \label{fig:toy-gen}
     \end{subfigure}
    \begin{subfigure}[t]{0.9\linewidth}
        \centering
        \includegraphics[width=\linewidth]{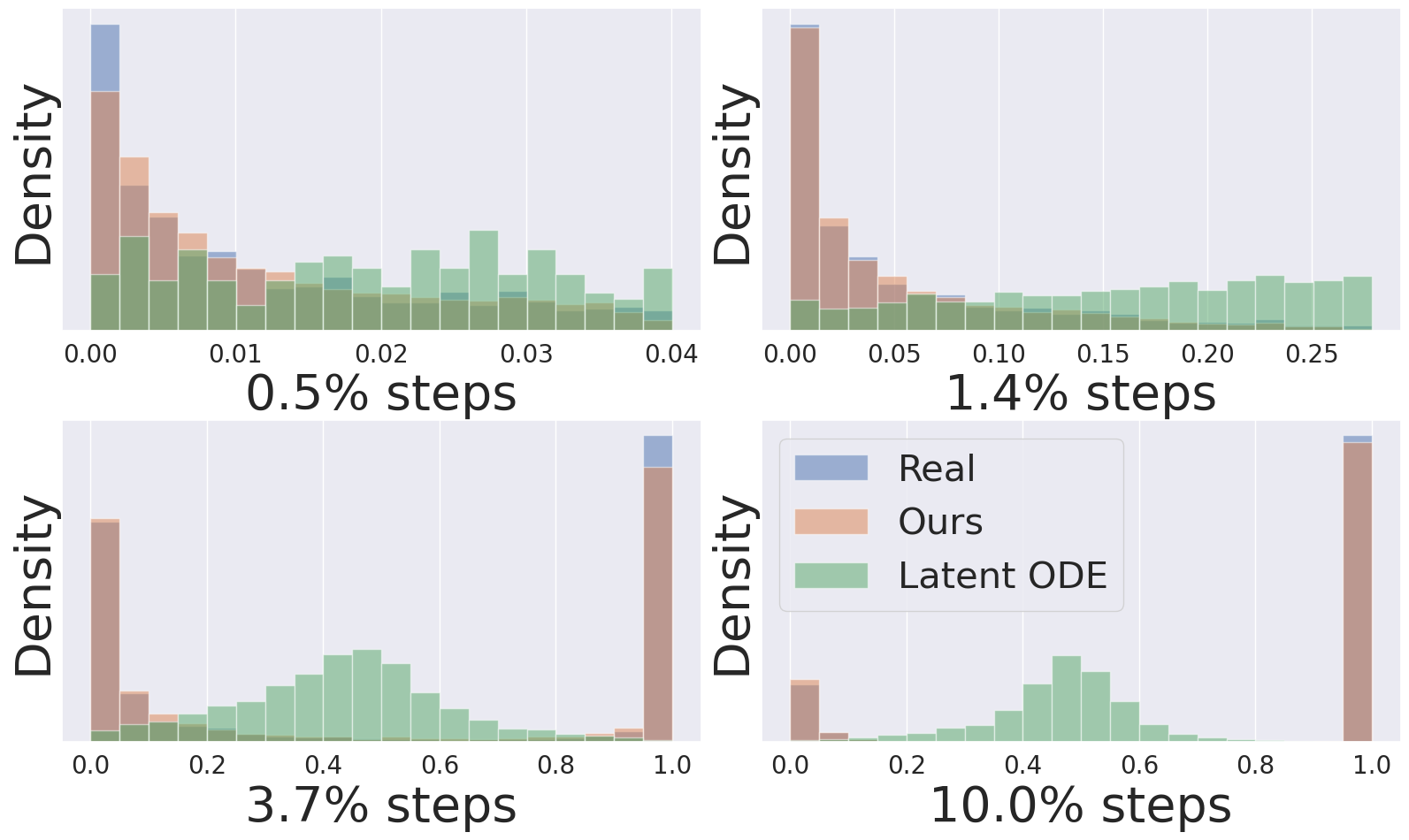}
        \caption{Marginal histograms at steps equally spaced between the $0.5\%$ and $10\%$ steps in log scale.}
        \label{fig:toy-marginal}
     \end{subfigure}
     \vspace{-2mm}
\end{figure}

\subsection{Generation with Real Time-Series Datasets}

We investigate the generative capability of \model{} on real time-series data. We show that our model is capable of fitting a wide variety of time-series data with significantly different temporal dynamics. We leave implementation details to \Appref{app:monash}.

\mypara{Data.} We use Monash Time Series Repository \citep{godahewa2021monash}, a comprehensive benchmark containing $30$ time-series datasets collected in the real world, and we choose FRED-MD, NN5 Daily, Temperature Rain, and Solar Weekly as our target datasets. We select these datasets based on average 1-lag autocorrelation metric, which measures 1-step correlation in time, to demonstrate a variety of temporal dynamics. The selected datasets have 1-lag ranging from 0.38 to 0.98, and this wide range indicates the variety of temporal dynamics that makes generative learning a challenging task. A sample from each dataset can be visualized in \Figref{fig:monash}.
\begin{figure}[h]
    \centering
    \includegraphics[width=\linewidth]{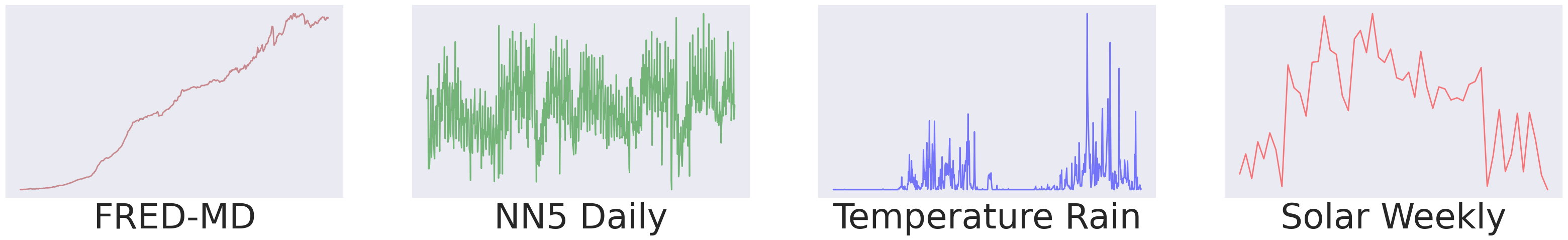}
    \caption{Monash data. The selected datasets exhibit a variety of temporal dynamics ranging from relatively smooth to stiff transitions.}
    \label{fig:monash}\vspace{-5mm}
\end{figure}

\mypara{Metrics.} We propose 3 different metrics for measuring generation performance, namely \textit{Marginal}, \textit{Classification}, and \textit{Prediction} scores. \textit{Marginal} scores calculate the absolute difference between empirical probability density functions of two distributions -- the lower the better \citep{ni2020conditional}. Following \citet{kidger2021neural}, we define \textit{Classification} scores as using a sequence model to classify whether a sample is real or generated and use its cross-entropy loss as a proxy for generation quality -- the higher the less distinguishable the samples, thus better the generation. Similarly, \textit{Prediction} scores use a train-on-synthetic-test-on-real seq2seq model to predict $k$ steps into the future -- the lower the more predictable, thus better the generation. We use a 1-layer S4 \citep{gu2021efficiently} for both Classification and Prediction scores (see \Appref{app:monash} for more details). We will discuss the necessity of all 3 metrics in the following discussion section.

We compare our model with several time-series generative models, namely the VAE-based methods such as RNN-VAE \citep{rubanova2019latent}, GP-VAE \citep{fortuin2020gp}, ODE$^2$VAE \citep{yildiz2019ode2vae}, Latent ODE \citep{rubanova2019latent}, GAN-based methods such as TimeGAN \citep{yoon2019time} and SDE GAN \citep{kidger2021neural}, and SaShiMi \citep{goel2022s}. Note that SaShiMi in this case is equivalent to S4 \citep{gu2021efficiently} as an autoregressive model with no latent variables, and suits as our closest baseline. To modify SaShiMi for general time-series, we modify its output to follow Gaussian probability with fixed variance instead of discrete tokens as in audio.  We show quantitative results in \Tabref{tab:gen}.

Our model significantly outperforms the baselines on all datasets. We note that baseline models have a hard time modeling NN5 Daily and Temperature Rain where the transition dynamics are stiff. For Temperature Rain where most data points lie around $x$-axis with sharp spikes throughout, latent ODE generates mostly closely to the $x$-axis, thus achieving lower marginal scores, but its generation is easily distinguishable from data, thus making it a less favorable generative model. We demonstrate that Marginal scores alone are an insufficient metric for generation quality. SaShiMi, an autoregressive model based on S4, does not perform as well on time series generation in the tasks considered. We further discuss the reason in \Appref{app:monash}.

\subsection{Interpolation \& Extrapolation}

\begin{table*}[t]
%
\centering
\resizebox{\linewidth}{!}{
\begin{tabular}{lll|ccccccc}
\toprule
 Metric & Task  & Data & RNN & RNN-VAE & ODE-RNN & GRU-D & Latent ODE & CRU & \textbf{\model{} (Ours)}  \\
\midrule
\multirow{4}{*}{MSE $(\times 10^{-3})\downarrow$} &\multirow{2}{*}{Interp.} & Physionet &  2.918    &   5.930   & 2.234  & 3.325   & 8.341 & 1.82 & \cellcolor{blue!10}\textbf{0.6287}\\
 & &  USHCN &  4.322 & 7.561 & 2.473 &3.395 &6.859  &0.16   & \cellcolor{blue!10}\textbf{ 0.0594}\\

\cmidrule{2-10}
&\multirow{2}{*}{Extrap.} & Physionet & 3.406 & 3.064  & \cellcolor{blue!10}\textbf{ 3.014 } &  3.120&  4.212 & 6.29 &  4.942\\
&&  USHCN &  9.474&9.083  & 9.045  &8.964  & 8.959  &  12.73& \cellcolor{blue!10}\textbf{2.975}\\
\midrule

\multirow{4}{*}{CRPS $(\times 10^{-2}) \downarrow$} &\multirow{2}{*}{Interp.} & Physionet &  2.09   &  5.59  & 2.40 & 2.71   & 6.16 & - & \cellcolor{blue!10}\textbf{1.25}\\
& & USHCN &  3.33   &  4.68 & 3.18 & 4.69 & 4.68 & - &  \cellcolor{blue!10}\textbf{0.438} \\

\cmidrule{2-10}
&\multirow{2}{*}{Extrap.} & Physionet &  3.30   &  2.17 & \cellcolor{blue!10}\textbf{2.16} &   13.9 &  2.43 & - &  2.36\\
&&  USHCN &   72.1  &  5.01 &  5.09 &  5.01  & 5.04  &  - & \cellcolor{blue!10}\textbf{2.76}\\
\midrule
\bottomrule 
\end{tabular}
}
\caption{Interpolation and extrapolation MSE $(\times 10^{-3})$ scores and CRPS $(\times 10^{-2})$ scores. Lower scores are better.}\label{tab:interpextrap}
\end{table*}

We also show that our model is expressive enough to fit to irregularly-sampled data and performs favorably in terms of interpolation and extrapolation. Interpolation refers to the task of predicting missing data given a subset of a sequence while extrapolation refers to the task that data is separated into 2 segments each with half the length of the full sequence, and one predicts the latter segment given the former.

\mypara{Data.} Following \citet{rubanova2019latent, schirmer2022modeling}, we use USHCN and Physionet as our datasets of choice. The United States Historical Climatology Network (USHCN) \citep{menne2015long} is a climate dataset containing daily measurements form 1,218 weather stations across the US for precipitation, snowfall, snow depth, minimum and maximum temperature. Physionet \citep{silva2012predicting} is a dataset containing health measurements of 41 signals from 8,000 ICU patients in their first 48 hours. We follow preprocessing steps of \citet{schirmer2022modeling} for training and testing. 

\mypara{Metrics.} We use mean squared error (MSE) and continuous ranked probability score (CRPS) to evaluate both interpolation and extrapolation. MSE determines the absolute prediction error while CRPS measures the difference of output cumulative distributions (CDF) against ground-truth. For a point observation, the CDF is assumed to be a step function.

We compare our model with RNN \citep{rumelhart1985learning}, RNN-VAE \citep{chung2014empirical, rubanova2019latent}, ODE-RNN \citep{rubanova2019latent}, GRU-D \citep{rubanova2019latent}, Latent ODE \citep{chen2018neural, rubanova2019latent}, and CRU\footnote{Numbers are taken from the original paper. We keep the significant digits unchanged.}\citep{schirmer2022modeling}. Results are shown in \Tabref{tab:interpextrap}.

We observe that our model outperforms previous continuous-time methods. Our model performs less well on extrapolation for Physionet compared to ODE-RNN and latent ODE and we postulate that this is due to the high variability granted by our latent space. Since new latent variables are generated as we extrapolate, our model generates paths that are more flexible (hence less predictable) than those of Latent ODE, which instead uses a single latent variable to encode an entire path. We also examine \model{}'s ability to perform probabilistic forecasting with CRPS. Notice that MSE and CRPS are correlated. This is due to the output probability being parameterized by Gaussian with fixed variance for all models, such that the difference between the Gaussian and step CDF is correlated with the difference between their mean. We additional observe that our model achieves ELBO of $-669.0$ and $-250.2$ on Physionet interpolation and extrapolation tasks respectively. These bounds are much tighter lower bounds than other VAE-based methods, \ie RNN-VAE, which reports $-412.8$ and $-220.2$, and latent ODE, which reports $-410.3$ and $-168.5$. We leave additional ELBO comparisons in \Appref{app:exp}.

\subsection{Runtime}

We additionally verify the computational efficiency of our model for both training and inference. We do so by training and inferring on synthetic data with controlled lengths specified below.

\mypara{Data.} We create a set of synthetic datasets with lengths $\{80, 320, 1280, 5120, 20480\}$ to investigate scaling of training/inference time with respect to sequence length. Training is done with 100 iterations through the synthetic data, and inference is performed given one batch of synthetic data (see \Appref{app:runtime}).

\mypara{Metrics.} We use wall-clock runtime measured in milliseconds.
\begin{figure}[H]
    \centering    \includegraphics[width=\linewidth]{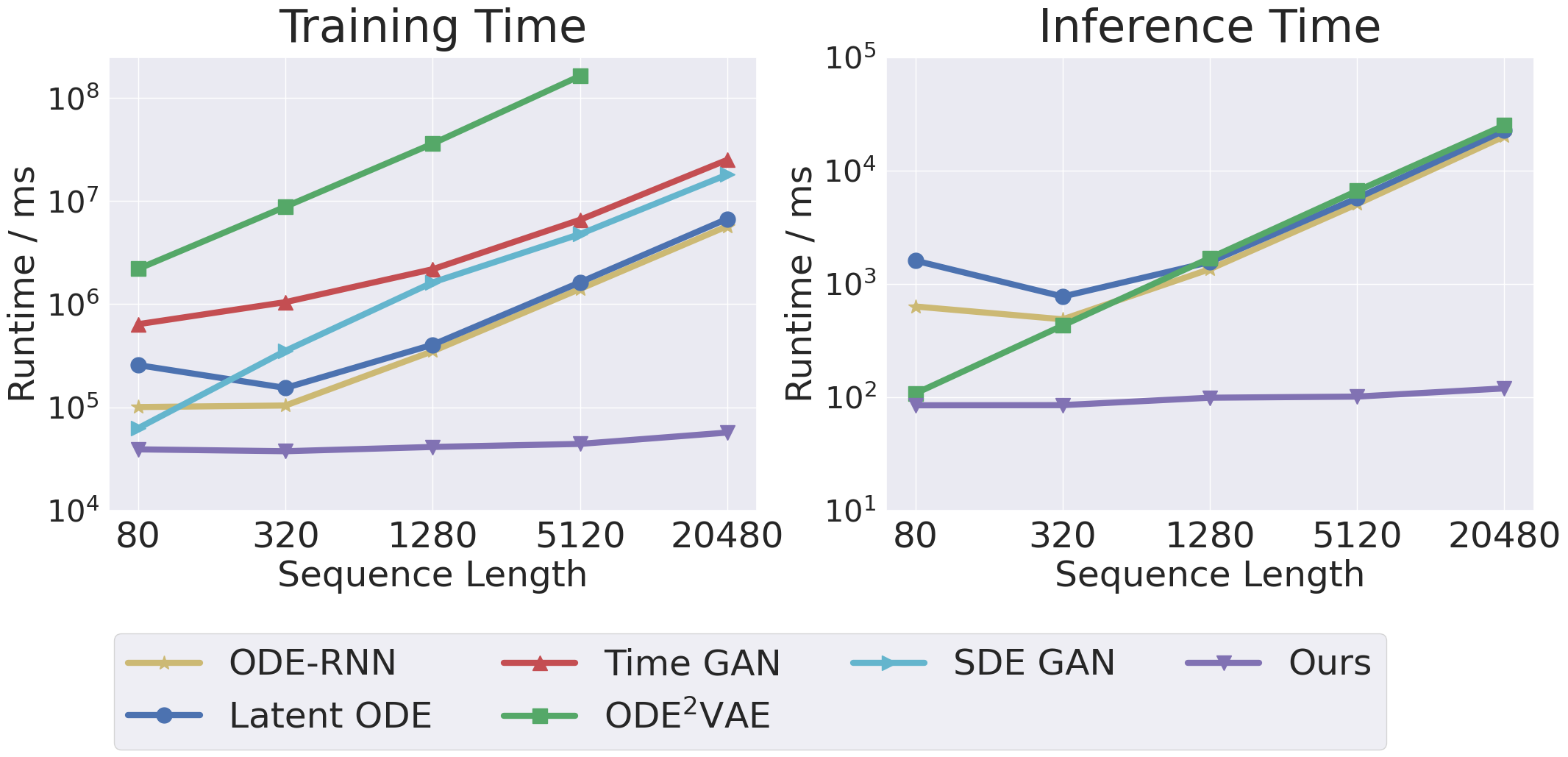}
     \caption{Runtime comparison. The $y$-axis shows run-time (\textbf{ms}) of each setting in log scale. Our runtime stays consistently lower across all sequence lengths investigated.}
     \label{fig:runtime}
     \vspace{-5mm}
\end{figure}
\Figref{fig:runtime} shows model runtime in log scale. ODE$^2$VAE fails to finish training on the last sequence length within a reasonable time frame, so we omit its plot of the last data point. Our model performs consistently and significantly lower than baselines, which are observed to scale linearly with input lengths, and is $\times 100$ faster than baselines in both training and inference on $20480$ length. 
\section{Conclusion}

We introduce \model{}, a powerful generative model with latent space evolution following a state space ODE. Our model is built with a deep stack of \model{} prior/generative/inference blocks, which are trained via standard sequence VAE objectives. We also show that under specific choices of model parameters, \model{} subsumes autoregressive S4 models. Experimentally, we demonstrate the modeling power of \model{} on datasets with a wide variety of temporal dynamics and show significant improvement in generation/interpolation/extrapolation quality. In addition, our model shows $\times 100$ speed-up in training and inference time on long sequences. \model{} demonstrates improved expressivity and computational efficiency, and we believe that it has a further role to play in modeling general time-series sequences.

\bibliography{reference}
\bibliographystyle{icml2023}

\newpage

\onecolumn
\begin{center}
    \huge\textbf{Deep Latent State Space Models for Time-Series Generation}
\end{center}
\vspace*{3mm}
\appendix
\parttoc

\section{Derivations}

\subsection{Computational Complexity of Vanilla Recurrent Representation}\label{app:der-recurr}
Assuming recurrence of the simplest form in \Eqref{eq:dis-ssm}, fulling computing matrix multiplication $\bar{\mA}\bar{h}_{t_k}$ requires $\mathcal{O}(N^2)$. Fully computing all hidden states sequentially requires $\mathcal{O}(N^2L)$. In space, saving each hidden state requires $\mathcal{O}(N)$ and in total requires $\mathcal{O}(NL)$.

\section{Proof}\label{app:proof}

To prove this result in \propref{prop:subsume}, we first prove the following proposition.
\begin{restatable}[]{proposition}{expressivity}\label{prop:expressivity} (Expressivity.)
Given any deep autoregressive S4 model $r: (x_{<t_n}, t_n) \mapsto y_{t_n}$ evaluated at time $t_n$ given input sequence $x_{<t_n}$, there exists a choice of $\theta$ such that $\mu_{x,n}(x_{<t_n}, 0, \theta) = r(x_{<t_n}, t_n)$.
\end{restatable}
\begin{sproof}
Consider SSM of the form in \Eqref{eq:multi-ssm} as a building block to our generative model with parameter $\theta$. We can choose $\mE = \mF = 0$ for all layers, which exactly reduces it to the SSM of an S4 model. Keeping all other hyperparameters (\eg non-linearities, number of stacking layers) the same, the final model is exactly the same as a deep autoregressive S4 model.
\end{sproof}

Now we give a proof sketch to \propref{prop:subsume}, 
\subsume*
\begin{sproof}
    From \propref{prop:expressivity} we know that we can choose $\theta$ so that $p_\theta(x|z) = p_\theta(x) = r(x)$ for all $z$, i.e., choose a decoder that ignores the latent variables $z$ and uses the same autoregressive structure over the observed variables as $r(x)$. This implies the posterior $p_{\theta, \lambda}(z|x)$ is equal to the prior $p_\lambda(z)$. We can then choose $\lambda$ and $\phi$ so that $p_\lambda(z) = \mathcal{N}(0,I)$ and $q_\phi(z|x)=\mathcal{N}(0,I)$ for all $x$. 
\end{sproof}

\efficiency*

\begin{proof}
    Recall \textit{SISO} SSM of the form
    \begin{align}
    \begin{split}
        \frac{d}{dt}\bm{h}_t &= \mA \bm{h}_t + \mB x_t\\
        y_t &= \mC \bm{h}_t + \mD x_t
    \end{split}
    \end{align}
    The calculation of $(y_{t_0}, \dots, y_{t_L})$ involves materializing the convolution filter, which can be calculated in $\mathcal{O}((L+N)\log(L+N))$ time and $\mathcal{O}(L + N)$ space for diagonal-plus-low-rank matrices \citep{gu2021efficiently}. Since the convolution is constant time in frequency domain, another computation cost comes from Fast Fourier Transform (FFT) and its inverse, which is $\mathcal{O}(L\log L)$ in time. The computation scales linearly with heads, Thus, a multi-input-multi-output (MIMO) SSM with $H$ heads can be processed in $\mathcal{O}(H(L+N)\log(L+N))$ time and $\mathcal{O}(H(L + N))$ space.
\end{proof}



\section{Architecture} \label{app:arch}

We parametrize our models using a similar architecture as in \citet{goel2022s}, but there is no pooling operation because for general time-series the time length is hardly divisible by a reasonable factor. Before we present the full structure, we present how a multi-channel inputs are parametrized (in the case of \model{} prior layer (\ref{eq:prior_layer}):
\begin{lstlisting}[language=Python]
    def (*@\textcolor{blue}{LS4\_prior\_layer\_multi}@*)(z, psi):
        # z: (B, L, C) 
        for c in range(C):
            z[:,:,c] = LS4_prior_layer(z[:,:,c], *psi.LS4_params)
        z = linear(z) # (B, L, C)  channel-wise mixing
        return z
\end{lstlisting}
The for loop is presented for demonstration purposes. In practice, the channels can be processed in parallel.
\subsection{Prior Model}
We specify the parametrization of $\mu_{z,n}$ and $\sigma_{z,n}$ in pseudo-code as the outputs of the following functions
\begin{lstlisting}[language=Python]
    def (*@\textcolor{blue}{prior\_model}@*)(z, lambda):
        # z: (B, L, z_dim) this is for time [t_0, t_{n-1}]
        z = linear(z) # (B,L,H) encoding to H

        outputs = []
        outputs.append(z)
        for i in range(lambda.num_layers1):
            z = linear(z) #  (B, L, H) -> (B, L, 2H)
            outputs.append(z)
            
        for i in range(lambda.num_layers2):
            z = LS4_prior_block_multi(z, *lambda.LS4_params)
                    # (B, L, H) -> (B, L, H) multi-channel SSMs
            z = ResBlock(z) # this is a general 1 layer residual block
        z = z + outputs.pop()

        for i in range(lambda.num_layers1):
            z = z + outputs.pop()
            outputs.append(z)
            z = linear(z)  #  (B, L, 2H) -> (B, L, H)
            for i in range(lambda.num_layers2):
                z = LS4_prior_block_multi(z, *lambda.LS4_params) 
                    # (B, L, H) -> (B, L, H) multi-channel SSMs
                z = ResBlock(z) # this is a general 1 layer residual block
                
            z = z + outputs.pop()
        z = layernorm(z)
        z = linear(z) # (B,L,z_dim)
        mu_z = LS4_prior_block_multi(z, *lambda.LS4_params) # (B,L,z_dim)
        sigma_z = LS4_prior_block_multi(z, *lambda.LS4_params) 
                                                        # (B,L,z_dim)
        return mu_z, sigma_z
\end{lstlisting}

\subsection{Generative Model}
\begin{lstlisting}[language=Python]
    def (*@\textcolor{blue}{prior\_model}@*)(x, z, theta):
        # z: (B, L, z_dim) this is for time [t_0, t_{n-1}]
        z = linear(z) # (B,L,H) encoding to H
        x = linear(x) # (B,L,H) encoding to H

        outputs_x, outputs_z = [], []
        outputs_z.append(z)
        outputs_x.append(x)
        for i in range(lambda.num_layers1):
            z = linear(z) #  (B, L, H) -> (B, L, 2H)
            x = linear(x) #  (B, L, H) -> (B, L, 2H)
            outputs_z.append(z)
            outputs_x.append(x)
            
        for i in range(lambda.num_layers2):
            z, x = LS4_gen_block_multi(z, *lambda.LS4_params) 
                    # (B, L, H) -> (B, L, H) multi-channel SSMs
            zx = ResBlock(concat(z, x)) 
                    # this is a general 1 layer residual block
            z, x = split(z, x)
        z = z + outputs_z.pop()
        x = x + outputs_x.pop()

        for i in range(lambda.num_layers1):
            z = z + outputs_z.pop()
            x = x + outputs_x.pop()
            outputs_z.append(z)
            outputs_x.append(z)
            z = linear(z)  #  (B, L, 2H) -> (B, L, H)
            x = linear(x)  #  (B, L, 2H) -> (B, L, H)
            for i in range(lambda.num_layers2):
                x, z = LS4_gen_block_multi(x, z, *lambda.LS4_params)
                            # (B, L, H) -> (B, L, H) multi-channel SSMs
                zx = ResBlock(concat(z, x)) 
                    # this is a general 1 layer residual block
                z, x = split(z, x)
                
            z = z + outputs_z.pop()
            x = x + outputs_x.pop()
        x = layernorm(x)
        z = layernorm(z)
        x = linear(concat(x, z)) # (B,L,2H) -> (B,L,x_dim)
        return mu_x
\end{lstlisting}
In practice, we find that only using $z$ input for the entire generative model produces better generation better than including $x$. We hypothesize that $x$ presents too strong of a signal for the model to reconstruct, and so the model learns to ignore signals from $z$ in that case.

\subsection{Inference Model}

\begin{lstlisting}[language=Python]
    def (*@\textcolor{blue}{inference\_model}@*)(x, phi):
        # x: (B, L, x_dim) this is for time [t_0, t_{n-1}]
        x = linear(x) # (B,L,H) encoding to H

        outputs = []
        outputs.append(x)
        for i in range(phi.num_layers1):
            x = linear(x) #  (B, L, H) -> (B, L, 2H)
            outputs.append(x)
            
        for i in range(phi.num_layers2):
            x = LS4_inf_block_multi(x, *phi.LS4_params) 
                    # (B, L, H) -> (B, L, H) multi-channel SSMs
            z = ResBlock(z) # this is a general 1 layer residual block
        x = x + outputs.pop()

        for i in range(phi.num_layers1):
            x = x + outputs.pop()
            outputs.append(x)
            x = linear(x)  #  (B, L, 2H) -> (B, L, H)
            for i in range(phi.num_layers2):
                x = LS4_inf_block_multi(x, *phi.LS4_params) 
                        # (B, L, H) -> (B, L, H) multi-channel SSMs
                z = ResBlock(z) # this is a general 1 layer residual block
                
            x = x + outputs.pop()
        x = layernorm(x)
        x = linear(x) # (B,L,x_dim)
        mu_z = LS4_inf_block_multi(x, *phi.LS4_params) # (B,L,x_dim)
        sigma_z = LS4_inf_block_multi(x, *phi.LS4_params) 
                                                # (B,L,x_dim)
        return mu_z, sigma_z
\end{lstlisting}

\section{Experiments} \label{app:exp}

For all experiments we use AdamW optimizer with learning rate $0.001$. We use batch size $64$ and train for $7000$ epochs for FRED-MD, NN5 Daily, and Solar Weekly, $1000$ epochs for Temperature Rain, and $500$ epochs for Physionet and USHCN. The datasets are split into $80\%$ training data and $20\%$ testing data.

\subsection{MONASH Forecasting Repository} \label{app:monash}
\mypara{Data.} For all selected MONASH data, FRED-MD, NN5 Daily, and Solar Weekly are normalized per sequence such that each trajectory is centered at its own mean and normally distributed. We make this choice from the observation that for some datasets such as NN5 Daily the min and max can vary significantly across different data points such that normalizing sequences with dataset-wise statistics makes it difficult to learn the temporal dynamics, which would be on a widely different range. For Temperature Rain we squash each sequence into $[0, 1]$. This is due to the fact that the dataset is always positive and lands mostly around $x$-axis with sharp spikes in between. For the former 3 datasets, we do not use output activation while for the last, we use sigmoid as our activation.

\mypara{Hyperparameters.} For all MONASH experiments, we use AdamW optimizer with learning rate 0.001 and no weight decay. For each of prior/generative/inference model, we use 4 stacks for each for loop in the pseudocode. For each LS4 block, we use 64 as the dimension of $\rvh_t$ and 64 SSM channels in parallel, same as used in S4 and SaShiMi. Each residual block consists of 2 linear layers with skip connection at the output level where the first linear layer has 2 times output size as the input size and the second layer squeezes it back to the input size of the residual block. We generally find 5-dimensional latent space gives better performance than 1, and so uses this setting throughout. We also employ EMA for model weights and use 0.999 as the lambda value, but we do not find this choice crucial. We also use 0.1 as the standard deviation for the observation as this gives better ELBO than other choices we experimented with such as $1, 0.5, 0.01$. For baselines, we reuse the code from official repo and follow their suggestions for training. To keep representation power similar, we use the same size for the latent space (for latent variable models) and the same output standard deviation for ELBO evaluation.

\mypara{Evaluation.} For generation evaluation. The classification model and the prediction model uses a linear encoder and linear decoder with a single S4 layer in between. The S4 layer uses 16 hidden state dimensions. For classification model, encoder maps data dimension to 16 hidden state dimension, and averages over the sequence output from S4 layer before feeding into decoder that outputs logit for binary classification. We use cross entropy loss. For prediction model, we use the same linear encoder and a decoder that maps 16 hidden dimension to data dimension. We predict $k=10$ steps into the future. The evaluation models are trained using AdamW with $0.01$ learning rate for $100$ epochs with batch size $128$. We generate samples equal to the number of testing data, which together are used to train the two models.

\mypara{Additional discussion.} We also briefly discuss the surprising result that SaShiMi does not perform as well on general time-series generation. We speculate that not using a quantization scheme  to define discrete output conditionals, as standard in autoregressive models for e.g., audio and images, is the cause behind this drop in performance. \model{} does not requires quantization and sets best performance with a simple Gaussian conditional on the data space.

\mypara{Additional comparisons.} We additionally compare with two more relevant baselines \citep{fabius2014variational} and \citep{li2020scalable} present result in \Tabref{tab:add-gen}. 
We note that Latent SDE has an abnormally high classification score for Temperature Rain data, and demonstrate that this is when the classification score is not reliable. Upon visually examining generated results for Latent SDE (\Figref{fig:lsde-gen}) compared to ground-truths (\Figref{fig:tr-gt}), one can observe that the variation is extremely noisy around the x-axis and that the selected classifier is not powerful enough to capture the distinction from real data due to the considerable noise that exists in both generated and real data, resulting in high classification score. Marginal and predictive scores are much worse in comparison and are more indicative of generation quality.

\begin{table}[h]
\centering
\footnotesize
    \begin{tabular}{ll|cccccccc}
    \toprule
    Data & Metric & VRNN & Latent SDE  & \textbf{\model{} (Ours)} \\
    \midrule
    FRED-MD & Marginal $\downarrow$ & 0.165 & 0.122 &  \cellcolor{blue!10}\textbf{0.0221} \\
    & Class. $\uparrow$ & 0.000970 & \cellcolor{blue!10}\textbf{0.687} & 0.544 \\
    & Prediction $\downarrow$ & 0.371 & 1.62  & \cellcolor{blue!10}\textbf{0.0373}\\
    \midrule
    NN5 Daily & Marginal $\downarrow$ & 0.151  & 0.125 &  \cellcolor{blue!10}\textbf{0.00671}\\
    & Class. $\uparrow$ & 0.00176 &  0.601 & \cellcolor{blue!10}\textbf{0.636}  \\
    & Prediction $\downarrow$ & 1.22 &  0.957 & \cellcolor{blue!10}\textbf{0.241} \\
    \midrule 
    Temp Rain & Marginal $\downarrow$ & 1.20 & 0.999 &  \cellcolor{blue!10}\textbf{0.0834} \\
    & Class. $\uparrow$& 0.479 & \cellcolor{blue!10}\textbf{14.534} & 0.976 \\
    & Prediction $\downarrow$ & 0.864 & 1.798 & \cellcolor{blue!10}\textbf{0.521}  \\
    \midrule 
    Solar Weekly & Marginal $\downarrow$& 0.297 & 0.234 &\cellcolor{blue!10}\textbf{0.0459} \\
    & Class. $\uparrow$ & 0.00164 & 0.764  & \cellcolor{blue!10} \textbf{0.683}  \\
    & Prediction $\downarrow$ & 0.964 & 1.01 & \cellcolor{blue!10}\textbf{0.141}\\
    \bottomrule 
    \end{tabular}
%

\caption{Additional generation results on FRED-MD, NN5 Daily, Temperature Rain, and Solar Weekly.}\label{tab:add-gen}
\end{table}

\begin{figure}[h]
    \centering
    \includegraphics[width=\linewidth]{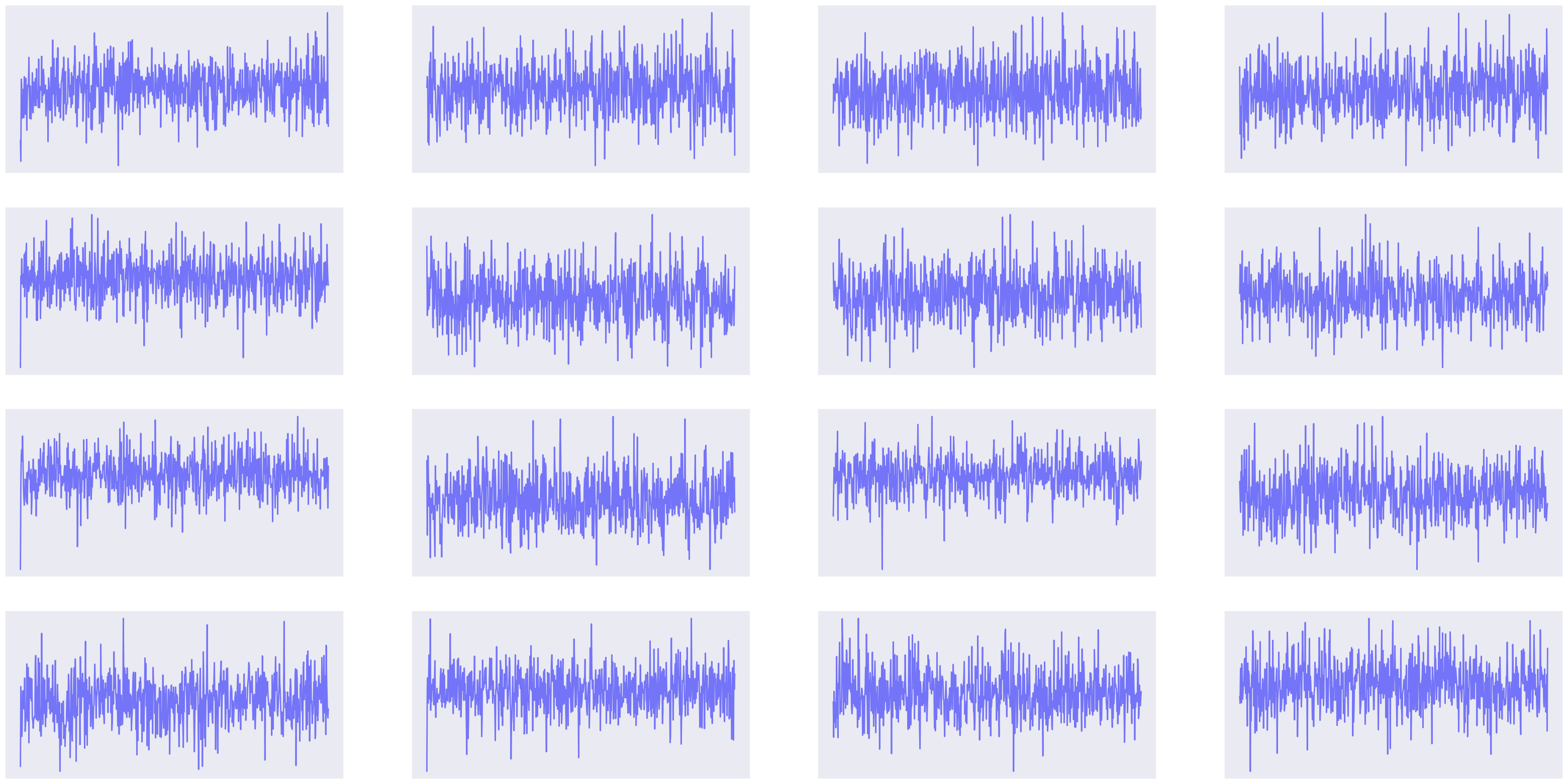}
    \caption{Latent SDE generation on Temperature Rain.}\label{fig:lsde-gen}
\end{figure}

\subsection{Physionet \& USHCN}
We follow the code provided by \citet{rubanova2019latent} to process Physionet and follow the code provided by \citet{de2019gru} for USHCN. For Physionet we do not use any activation to constrain the output space, and for USHCN, we use sigmoid activation for output.

We present the variational lowerbound results in \Tabref{tab:vlb}.
\begin{table}[]
    \centering

    \begin{tabular}{ll|cccccccc}
    \toprule
    Task  & Data &  RNN-VAE  & Latent ODE & \textbf{\model{} (Ours)}  & \textbf{\model{}\textsuperscript{IWAE} (Ours)}  \\
    \midrule
    \multirow{2}{*}{Interp.} & Physionet &   -412.8  &   -410.3   &  -669.0 & \cellcolor{blue!10}\textbf{-684.3}\\
    &  USHCN & -244.9   & -251.0   & -312.2 & \cellcolor{blue!10}\textbf{ -315.6}\\
    
    \midrule
    \multirow{2}{*}{Extrap.} & Physionet &-220.2  &   -168.5   & -250.2 & \cellcolor{blue!10}\textbf{-288.7}\\
    &  USHCN & -113.3  &  -110.3& -194.4 & \cellcolor{blue!10}\textbf{-211.8}\\
    \midrule
    \bottomrule 
    \end{tabular}
    \caption{ELBO comparisons with VAE-based models.}
    \label{tab:vlb}
\end{table}

\mypara{Hyperparameters.} In general we keep the hyperparameter choices the same as in MONASH, and we describe a few differences for these 2 datasets. For USHCN, we use 10 as the dimension for latent space, same as in \citet{rubanova2019latent} and we use Sigmoid as the output activation with output standard deviation 0.01. For Physionet, we use no output activation and 0.05 standard deviation, and use 20-dimensional latent space, same as in baselines.

\subsection{Runtime} \label{app:runtime}

We test all models on a single RTX A5000 GPU. To set up the dataset, we need to fully populate the GPU for each dataset during training for our benchmarking. For sequence lengths $\{80, 320, 1280, 5120, 20480\}$, we build dataset of length $\{102400, 25600, 6400, 1600, 400\}$ each with batch size $\{1024, 256, 64, 16, 4\}$ so that for each dataset the models are trained with $100$ iterations.

\section{Generation Results}

We present ground-truths and generations on the two hardest selected datasets, NN5 Daily and Temperature Rain because these are the hardest to model.

\begin{figure}[h]
    \centering
    \includegraphics[width=\linewidth]{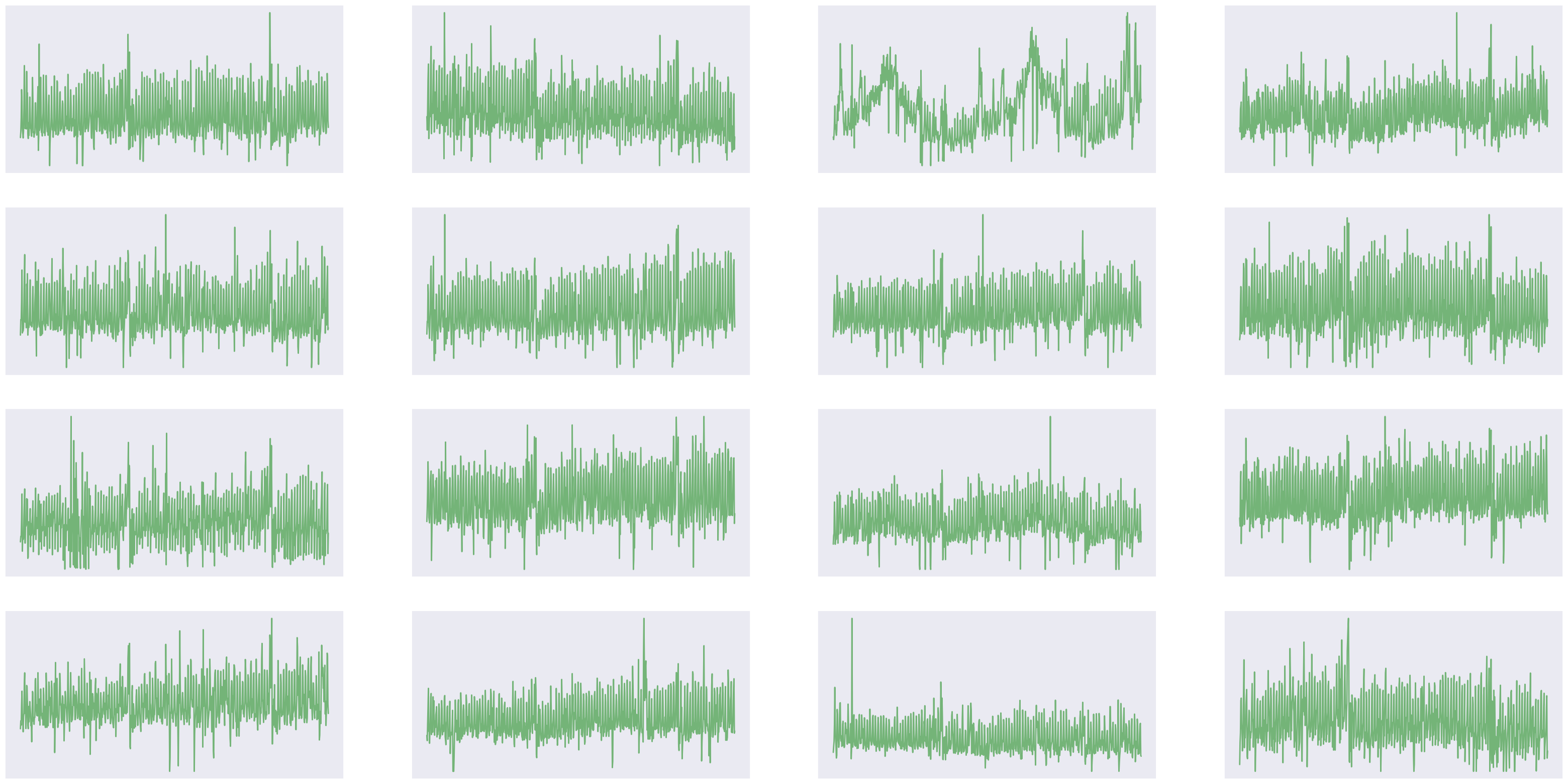}
    \caption{Normalized NN5 Daily data.}
\end{figure}

\begin{figure}[h]
    \centering
    \includegraphics[width=\linewidth]{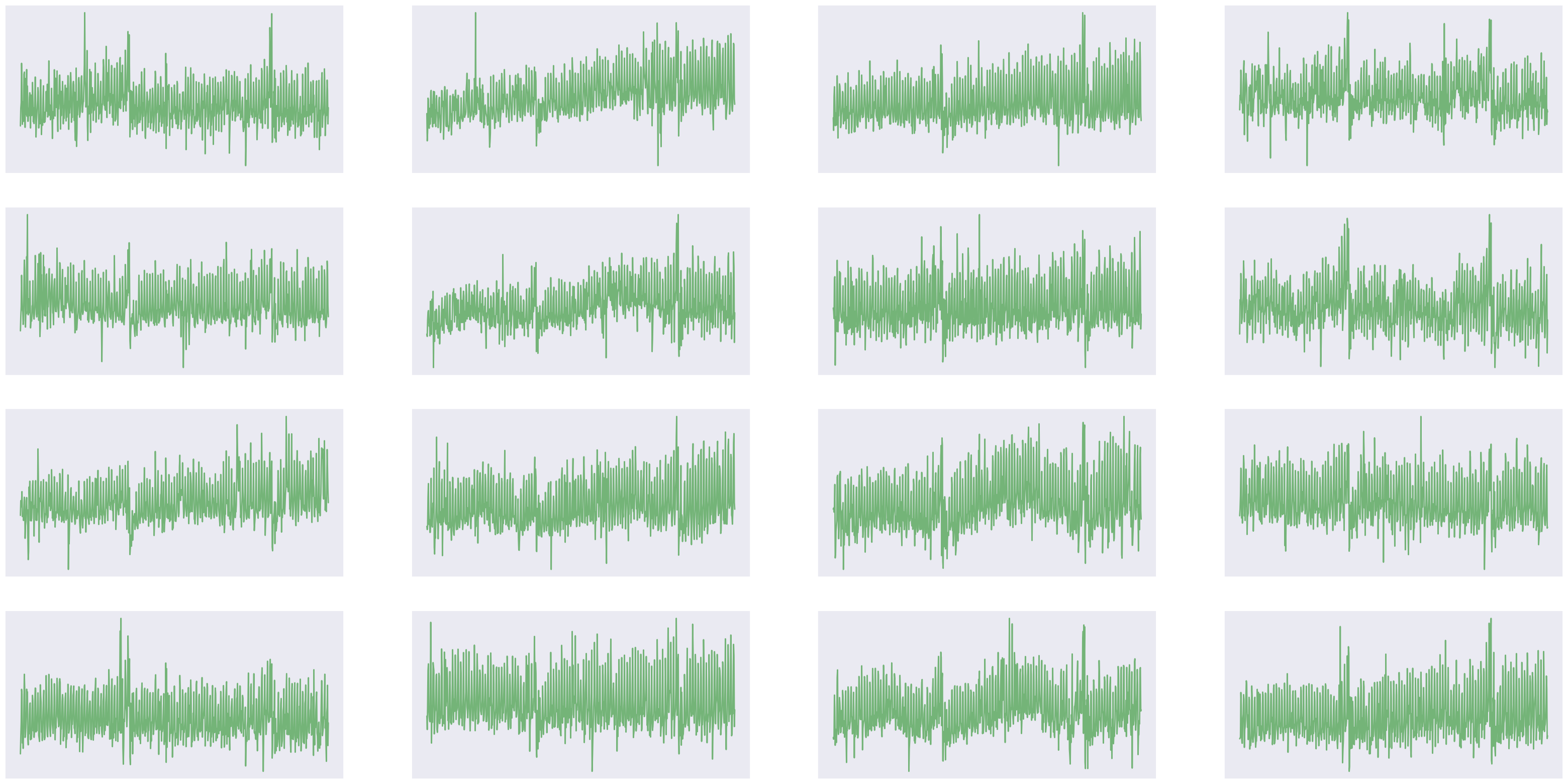}
    \caption{NN5 Daily generation.}
\end{figure}

\begin{figure}[h]
    \centering
    \includegraphics[width=\linewidth]{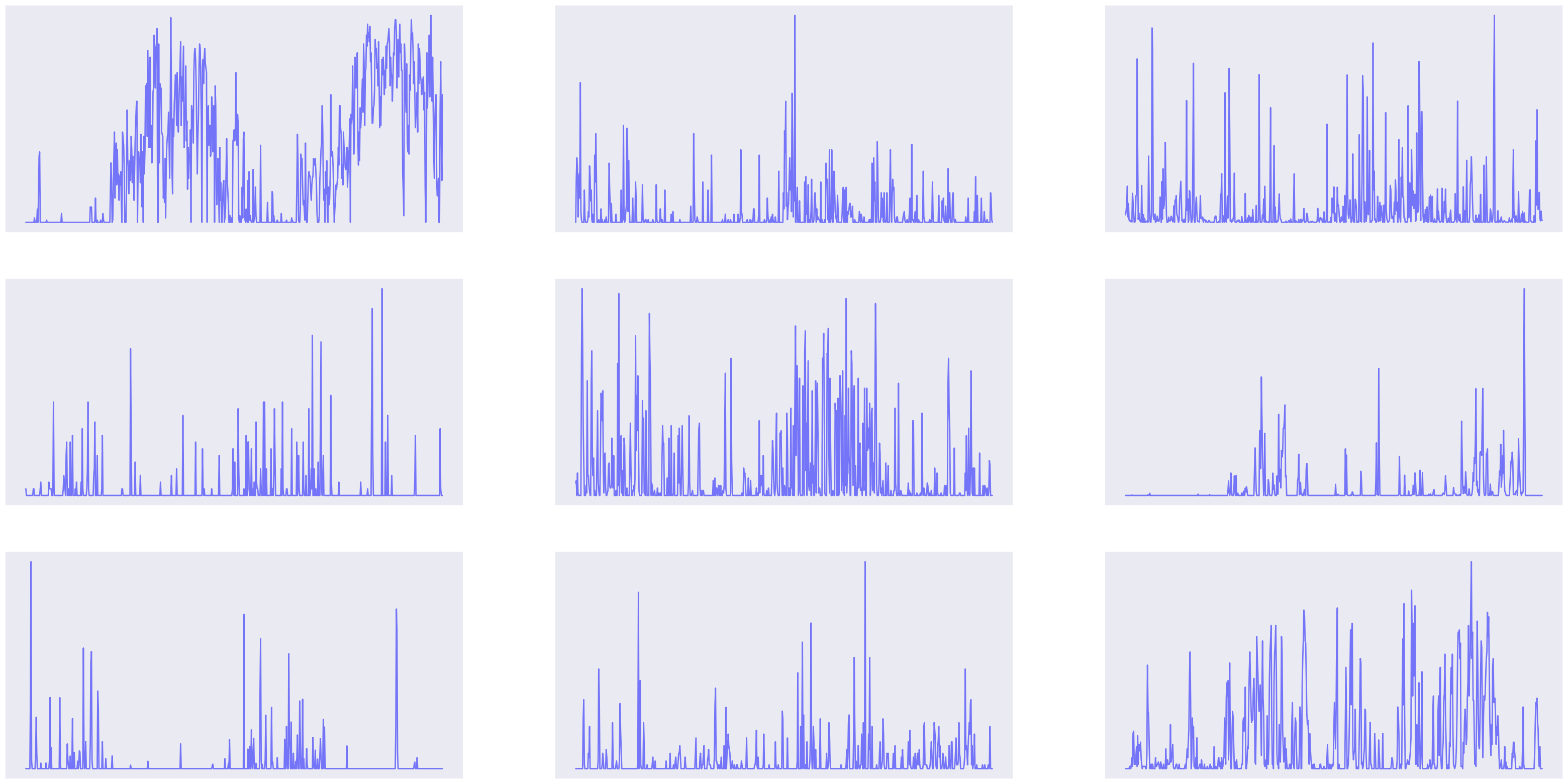}
    \caption{Normalized Temperature Rain data.}\label{fig:tr-gt}
\end{figure}

\begin{figure}[h]
    \centering
    \includegraphics[width=\linewidth]{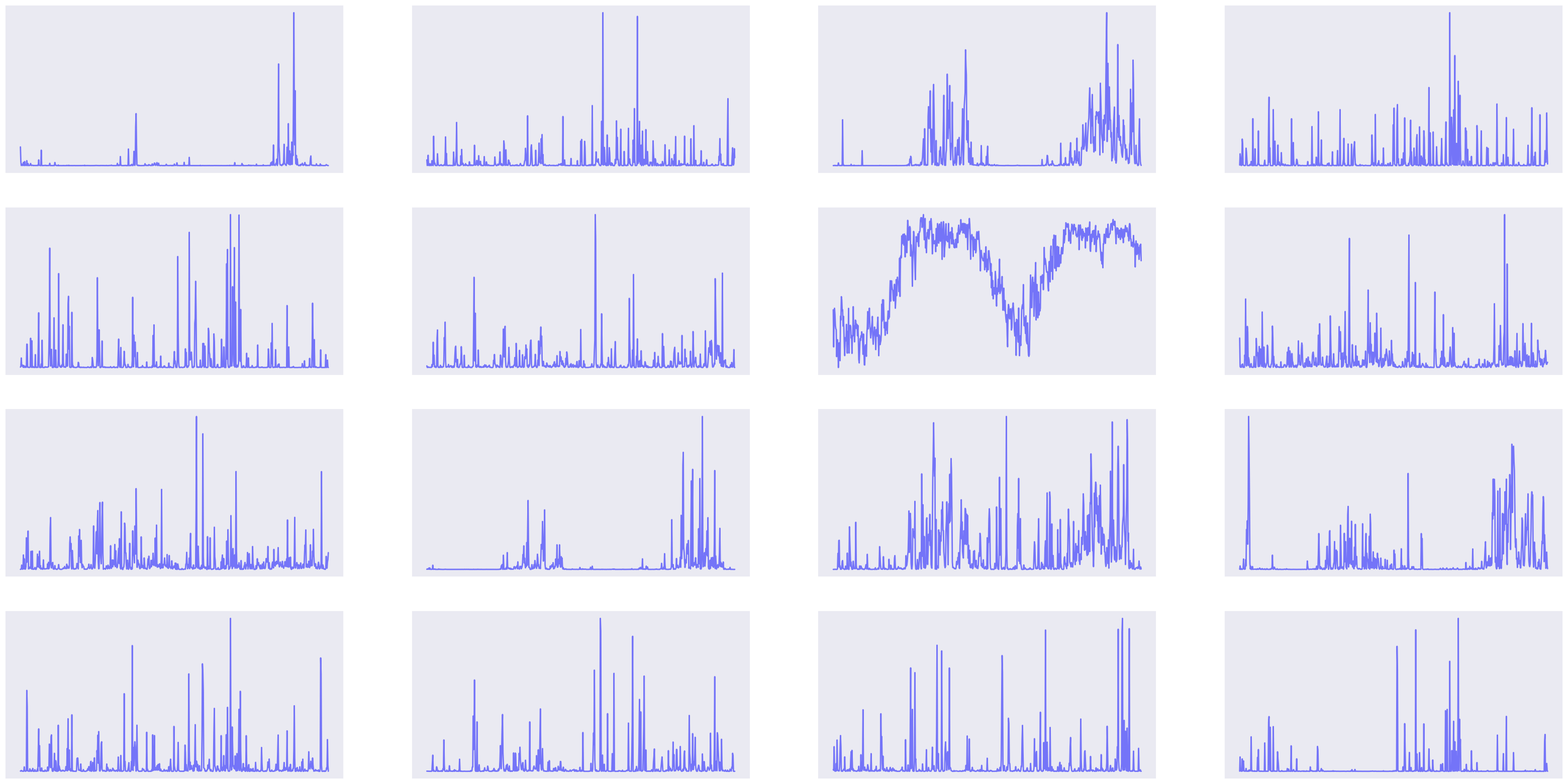}
    \caption{Temperature Rain generation.}
\end{figure}

\end{document}